\documentclass[10pt,twocolumn,letterpaper]{article}

\usepackage{dicta}
\usepackage{times}
\usepackage{epsfig}
\usepackage{graphicx}
\usepackage{amsmath}
\usepackage{amssymb}
\usepackage{balance}

\usepackage{amssymb,amsthm,mathtools}
\newtheorem{proposition}{Proposition}

\usepackage[pagebackref=true,breaklinks=true,letterpaper=true,colorlinks,bookmarks=false]{hyperref}

\usepackage{array}
\usepackage{booktabs}
\usepackage{fontawesome5}
\usepackage{multirow}

\usepackage{ragged2e}

\usepackage[table]{xcolor}
\usepackage{tikz}

\colorlet{mmwA}{magenta!85!purple}
\colorlet{mmwB}{magenta!75!purple}
\colorlet{mmwC}{magenta!60!blue}
\colorlet{mmwD}{purple!70!blue}
\colorlet{mmwE}{blue!75!purple}
\colorlet{mmwF}{blue!70!cyan}
\colorlet{mmwG}{cyan!75!blue}
\colorlet{mmwH}{cyan!85!blue}
\colorlet{mmwI}{cyan!90!blue}

\DeclareRobustCommand{\mmWaveColored}{%
  \textcolor{magenta!85!purple}{m}%
  \textcolor{magenta!60!blue}{m}%
  \textcolor{purple!70!blue}{W}%
  \textcolor{blue!70!cyan}{a}%
  \textcolor{cyan!75!blue}{v}%
  \textcolor{cyan!90!blue}{e}%
}

\newcommand{\coolname}{SGCA-Net}

\usepackage{pgfplots}
\usepgfplotslibrary{groupplots}
\pgfplotsset{compat=1.18}

\dictafinalcopy %

\def\dictaPaperID{100} %

\ifdictafinal\fi

\newlength{\copyrightleft}
\begin{document}

\title{Automotive {\texttt{\mmWaveColored}} Spinning Radar Place Recognition\\ with Spatially Gated Feature-Correlation Representation}

\author{
Saimunur Rahman$^{1}$,
Sagun Singh Shrestha$^{1}$,
Abdelwahed Khamis$^{1,2}$,
Peyman Moghadam$^{1,2}$\\
{\small $^{1}$CSIRO Robotics, CSIRO, Australia \quad
\texttt{firstname.lastname@csiro.au}}\\
{\small $^{2}$Queensland University of Technology, Australia \quad
\texttt{firstname.lastname@qut.edu.au}}
}

\maketitle

\begin{tikzpicture}[remember picture,overlay]
\node[
    anchor=south west,
    text width=\textwidth,
    inner xsep=0pt,
    inner ysep=0pt,
    align=center,
    font=\fontsize{7.5}{8.2}\selectfont
] at ([xshift=\copyrightleft,yshift=0.25in]current page.south west) {%
    \noindent\justifying
    \textcopyright~2026 IEEE. Personal use of this material is permitted.
    Permission from IEEE must be obtained for all other uses, in any current
    or future media, including reprinting/republishing this material for
    advertising or promotional purposes, creating new collective works,
    for resale or redistribution to servers or lists, or reuse of any
    copyrighted component of this work in other works.
};
\end{tikzpicture}

\begin{abstract}

Automotive spinning FMCW radar provides dense, $360^\circ$ sensing and remains reliable under poor illumination and adverse weather, making it well-suited to autonomous navigation. Place recognition uses these observations to identify previously visited locations for re-localization and long-term navigation. However, heading changes appear as circular shifts in the polar radar representation, and conventional global aggregation can lose relationships among radar responses that are important for distinguishing similar places. We propose \coolname{}, a spinning radar place recognition framework that combines rotation-robust feature extraction with Spatially Gated Correlation Aggregation (SGCA). SGCA learns spatial weights to reduce the influence of unstable and ambiguous radar regions, while aggregating pairwise correlations among local responses to preserve informative feature relationships. Experiments on the MulRan dataset show that \coolname{} consistently outperforms SOTA methods across urban, campus, and open-road environments, while remaining robust to substantial heading variation. Evaluation on the HeRCULES dataset further demonstrates that \coolname{} generalizes to unseen environments and radar sensors without fine-tuning.

\end{abstract}

\section{Introduction}
\label{sec:introduction}

Place recognition (PR) determines whether a current observation corresponds to a previously visited place and is central to long-term autonomous robot navigation, loop closure, and re-localization~\cite{yin2025general, knights2023wild}. PR is commonly formulated as a retrieval problem, where observations are encoded as compact global descriptors and a query is matched to a database using nearest-neighbour search, with candidate matches subsequently refined through a re-ranking method~\cite{sgv}. The effectiveness of this retrieval formulation therefore depends on descriptors that remain consistent across revisits while providing sufficient discrimination between geographically distinct locations ~\cite{arandjelovic2016netvlad}. The ability to construct such robust descriptors depends strongly on the sensing modality.  Cameras \cite{lowry2015visual,hausler2021patch,hausler2025pair} and LiDAR~\cite{knights2023wild,vidanapathirana2022logg3d} have been widely adopted for PR, while automotive mmWave frequency-modulated continuous-wave (FMCW) radar offers complementary robustness under poor illumination and adverse weather conditions, including rain, fog, and dust~\cite{barnes2020oxford,burnett2022ready,hong2020radarslam,kim2020mulran}. Among automotive radar configurations, spinning FMCW radar is particularly suitable for global descriptor-based PR~\cite{jang2023raplace,gadd2024open,komorowski2021radarloc,kim2025sherloc} because it provides dense, 360$^\circ$ azimuth-range measurements, in contrast to the sparser and more limited field-of-view observations typically produced by 4D automotive radar sensors~\cite{cai2022autoplace,casado2024spr,zhou2022towards}. 

A spinning FMCW radar scan is formed by mechanically rotating the radar antenna over azimuth while recording reflected signal power over range, producing a 360$^\circ$ azimuth-range polar representation~\cite{barnes2020oxford,kim2020mulran}. 
While this representation provides dense coverage of the surrounding environment, it introduces several challenges for place recognition. Changes in vehicle heading appear as circular shifts along the azimuth dimension, while the measured radar responses vary with reflector geometry, material properties, incidence angle, multipath, sidelobes, and range~\cite{hong2020radarslam,burnett2022ready,barnes2020oxford}. In addition, dynamic objects, occlusions, and transient clutter can alter the spatial distribution of returns between revisits. Consequently, observations of geographically distinct places may contain similar local radar responses, while observations of the same place may exhibit substantial variation across traversals~\cite{jang2023raplace,gadd2024open,komorowski2021radarloc,kim2025sherloc}. An effective radar descriptor must therefore preserve repeatable spatial structure across revisits while reducing the influence of unstable and ambiguous responses.

Existing radar PR methods have primarily addressed heading variation through radar-specific convolutional processing, circular operations, transform-domain representations, or rotation-aware descriptor design~\cite{jang2023raplace,gadd2024open,komorowski2021radarloc,kim2025sherloc}. However, the subsequent aggregation of local radar features into a compact global representation remains comparatively less explored. First-order aggregation methods, including average, max, and generalized mean (GeM) pooling, primarily summarize feature presence or response magnitude. Learned aggregation approaches, such as NetVLAD~\cite{arandjelovic2016netvlad} and its radar variants~\cite{gadd2024open, kim2025sherloc}, provide more expressive representations by aggregating local features, but largely characterize the distribution of individual responses. For spinning radar, this can discard informative co-occurrence structure across the scan, where place identity may depend on joint response patterns rather than isolated feature strength. A global descriptor should therefore preserve not only salient individual responses, but also relationships among responses that remain consistent across revisits.

Motivated by these observations, we introduce \coolname{}, a learned spinning FMCW radar place recognition framework that combines a rotation-robust radar backbone with a Spatially Gated Correlation Aggregation (SGCA) descriptor. The backbone operates directly on the azimuth-range polar representation and extracts local features that are robust to heading-induced circular shifts. The SGCA module subsequently projects these features into a compact descriptor space, learns spatial keep/dustbin weights to reduce the contribution of unstable or ambiguous locations, and aggregates pairwise correlations among the reweighted local responses. Unlike conventional first-order aggregation, which mainly summarizes feature presence or response strength, SGCA therefore produces a spatially gated feature-correlation representation to preserve repeatable radar structure. Power normalization further stabilizes the representation. %
The contributions of this work are summarized as follows:

\begin{itemize}
    \item We propose \coolname{}, a spinning radar place recognition
    framework that combines rotation-robust feature extraction with
    spatially gated correlation aggregation.

    \item We introduce a spatially weighted feature correlation
    representation that captures pairwise relationships among radar responses
    while reducing the influence of unstable and ambiguous regions.

    \item We demonstrate consistent improvements across diverse MulRan
    environments and show that the representation learned by
    \coolname{} transfers effectively to unseen sensor and environments without fine-tuning.
\end{itemize}

\section{Related Work}
\label{sec:related_work}
\noindent \textbf{Radar Place Recognition.}
Radar Place Recognition (RPR) has evolved from hand-crafted representations towards learned global features. ScanContext~\cite{kim2018scan} represents a scan using a polar grid and handles heading variation through circular alignment, while Ring Key provides a compact rotation-invariant summary derived from the same representation. Although these methods are efficient and interpretable, their hand-crafted formulations offer limited adaptability to the substantial appearance variation present in automotive radar observations. RadarLoc~\cite{komorowski2021radarloc} showed that a compact radar-specific convolutional network operating directly on polar radar images can learn effective global descriptors for large-scale place recognition. Its cylindrical processing and GeM pooling provide robustness to heading variation while maintaining a lightweight architecture.
More recent methods have explored alternative feature aggregation,
matching, and sensing configurations. Open-RadVLAD
~\cite{gadd2024open} aggregates local radar features using a
VLAD-based representation, while RaPlace~\cite{jang2023raplace}
employs Radon-transform-based processing to improve robustness to
changes in radar appearance. SHeRLoc~\cite{kim2025sherloc} extends
radar place recognition to heterogeneous radar sensors and includes
evaluation on spinning radar observations. Despite these advances,
most learned spinning radar approaches still compress their feature
maps using conventional global aggregation. Consequently, descriptor
formation is dominated by the strength or frequency of local responses, while interactions among feature responses receive limited attention. This is restrictive for spinning radar, where a place may be identified more reliably by recurring combinations of returns than by isolated strong measurements.

\noindent \textbf{Global Feature Aggregation.}
Global aggregation converts local feature maps into compact descriptors for efficient place retrieval. First-order methods, including average, maximum, and generalized mean (GeM) pooling~\cite{radenovic2018fine}, summarize individual feature responses but discard relationships between feature channels. GeM is commonly used for place recognition and is adopted by RadarLoc for spinning radar PR. Learned aggregation methods such as NetVLAD~\cite{arandjelovic2016netvlad} and Open-RadVLAD~\cite{gadd2024open} provide more expressive descriptors by modelling the distribution of local features, but do not explicitly capture relationships among feature responses. This is limiting for automotive radar, where feature maps may contain empty regions, multipath, transient objects, and repeated structures that can introduce unstable or ambiguous responses.

\noindent \textbf{Pairwise Feature Correlation Aggregation.}
Pairwise correlation representations model dependencies between feature channels through bilinear or covariance statistics. Bilinear pooling~\cite{gao2016compact} and covariance-based representations~\cite{li2018towards} have shown that such interactions can improve discrimination beyond first-order aggregation. However, directly applying correlation aggregation to radar is problematic because unreliable spatial locations can influence the resulting correlation statistics. \coolname{} addresses this by applying learned spatial gating before correlation aggregation, reducing the contribution of unstable regions while retaining informative relationships among feature responses. Matrix square-root normalization~\cite{li2018towards} is then used to stabilize the resulting representation.

\begin{figure*}[t]
    \centering
    \includegraphics[width=1\linewidth]{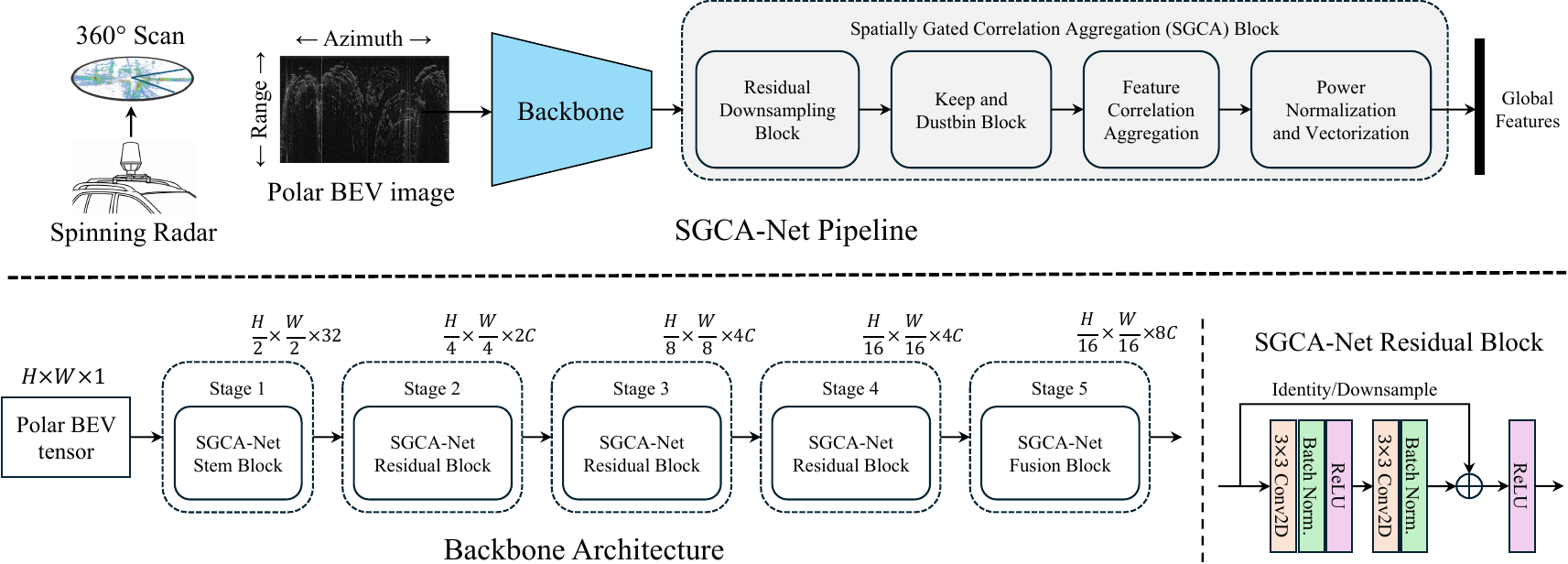}
    \caption{Overview of SGCA-Net. The input is a polar BEV radar image, with azimuth and range as image axes. The image is processed by a stem block and residual backbone stages that reduce spatial resolution while increasing feature channels. The resulting feature map is passed to the SGCA block, which performs residual downsampling, keep-and-dustbin selection, feature-correlation aggregation, power normalization, and vectorization to produce a global descriptor.}
    \label{fig:pipeline}
\end{figure*}

\section{Methodology}
\label{sec:methodology}

SGCA-Net is a deep metric learning framework for spinning FMCW radar place recognition. Given a polar radar image $\mathbf{X}\in\mathbb{R}^{H\times W\times1}$, the network extracts a dense feature map using a rotation-robust radar backbone and converts it into a compact global descriptor through the proposed SGCA pooling module, as illustrated in Fig.~\ref{fig:pipeline}. The complete descriptor extraction pipeline is expressed as $\mathbf{d}=f(\mathbf X)=\Psi(\Phi(\mathbf X))$, where $\Phi(\cdot)$ denotes the backbone, $\Psi(\cdot)$ denotes the SGCA module, and $\mathbf{d}\in\mathbb{R}^{D}$ is the final descriptor. The network is trained end-to-end with a rotation-robust metric learning objective, described in Sec. \ref{sec:Triplet}

\subsection{Radar Polar Image Backbone}

The backbone converts the range-azimuth radar scan $\mathbf{X}\in\mathbb{R}^{H\times W\times1}$ into a dense feature map for SGCA (shown in Fig.~\ref{fig:pipeline}). It consists of a convolutional stem, four residual stages, and a top-down fusion block. 
We denote a cylindrically padded $k\times k$ 2D convolution with stride $s$ by $\mathcal{C}_{k,s}(\cdot)$, batch normalization by $\mathcal{B}(\cdot)$, ReLU by $\sigma(\cdot)$, and transposed convolution by $\mathcal{T}_{k,s}(\cdot)$.

Since spinning radar measurements are periodic in azimuth, conventional zero padding introduces an artificial discontinuity at the $0^\circ$ and $360^\circ$ boundaries. We therefore apply cylindrical padding along the azimuth dimension while retaining zero padding along range: 
\begin{equation}
\tilde{\mathbf{X}}(r,c)=
\begin{cases}
\mathbf{X}(r,c+W), & c<0,\\
\mathbf{X}(r,c), & 0\le c<W,\\
\mathbf{X}(r,c-W), & c\ge W.
\end{cases}
\end{equation}
This allows convolutional kernels near the azimuth boundary to operate on physically continuous neighbourhoods, improving robustness to heading-induced circular shifts. The stem applies a $5\times5$ stride-2 convolution, batch normalization, and ReLU, giving $\mathbf F_0=\sigma(\mathcal B(\mathcal C_{5,2}(\mathbf X)))$, where $\mathbf F_0\in\mathbb{R}^{\frac{H}{2}\times\frac{W}{2}\times32}$.
The stem output is processed by four residual stages. Each residual block contains two $3\times3$ convolutions and a shortcut connection, as shown in bottom right of Fig.~\ref{fig:pipeline}. For an input feature map $\mathbf{U}$, the residual mapping is
\begin{equation}
\mathcal{R}(\mathbf{U}) =
\mathcal B\left(
\mathcal C_{3,1}\left(
\sigma\left(
\mathcal B\left(
\mathcal C_{3,s}(\mathbf{U})
\right)\right)\right)\right),
\end{equation}
and the block output is $\mathbf{Y}=\sigma(\mathcal{R}(\mathbf{U})+\mathcal{S}(\mathbf{U}))$, where
\begin{equation}
\mathcal S(\mathbf{U})=
\begin{cases}
\mathbf{U} & \text{if dimensions are unchanged},\\
\mathcal B\left(\mathcal C_{1,2}(\mathbf{U})\right), & \text{otherwise}.
\end{cases}
\end{equation}

The projection shortcut is used when spatial downsampling or channel expansion is required. The four stages produce
\begin{equation}
\begin{aligned}
\mathbf F_1&\in\mathbb{R}^{\frac{H}{4}\times\frac{W}{4}\times32},&
\mathbf F_2&\in\mathbb{R}^{\frac{H}{8}\times\frac{W}{8}\times64},\\
\mathbf F_3&\in\mathbb{R}^{\frac{H}{16}\times\frac{W}{16}\times128},&
\mathbf F_4&\in\mathbb{R}^{\frac{H}{32}\times\frac{W}{32}\times256}.
\end{aligned}
\end{equation}

The fusion block combines the deepest feature map with the intermediate feature map $\mathbf F_3$. The deepest feature is projected and upsampled as
\begin{equation}
\hat{\mathbf P}_3=
\mathcal T_{2,2}\left(\mathcal C_{1,1}(\mathbf F_4)\right),
\quad
\hat{\mathbf P}_3\in\mathbb{R}^{\frac{H}{16}\times\frac{W}{16}\times256}.
\end{equation}
The lateral feature is projected as $\mathbf L_3=\mathcal C_{1,1}(\mathbf F_3)$, with $\mathbf L_3\in\mathbb{R}^{\frac{H}{16}\times\frac{W}{16}\times256}$. The fused backbone output is
\begin{equation}
\mathbf F=\hat{\mathbf P}_3+\mathbf L_3,
\quad
\mathbf F\in\mathbb{R}^{\frac{H}{16}\times\frac{W}{16}\times256}.
\end{equation}
This fused tensor is passed to the SGCA module to form the final global descriptor. For notational simplicity, we define the spatial dimensions of the fused backbone feature map as $h = H/16$ and $w = W/16$, such that $\mathbf{F} \in \mathbb{R}^{h \times w \times 256}$.

\subsection{Spatially Gated Correlation Aggregation}

The fused backbone output $\mathbf{F}\in\mathbb{R}^{h\times w\times 256}$ is first projected into a compact feature space for descriptor construction. Let $\mathcal{P}_s(\cdot)$, $\mathcal{P}_1(\cdot)$, and $\mathcal{P}_2(\cdot)$ denote $1\times1$ projection convolutions, and let $\rho(\cdot)$ denote the GELU activation. The projected feature map $\mathbf{Z}\in\mathbb{R}^{h\times w\times64}$ is computed as

\begin{equation}
\mathbf{Z}
=
\mathcal{B}\!\left(\mathcal{P}_s(\mathbf{F})\right)
+
\mathcal{B}\!\left(
\mathcal{P}_2
\!\left(
\rho
\!\left(
\mathcal{B}
\!\left(
\mathcal{P}_1(\mathbf{F})
\right)
\right)
\right)
\right).
\label{eq:projection}
\end{equation}

This residual bottleneck reduces the channel dimension from 256 to 64 while preserving backbone information through the residual branch, producing a feature basis tailored for descriptor formation.

A learned keep/dustbin gate then estimates the spatial importance of each projected local feature. A $1\times1$ convolution predicts two logits at every spatial location,
\begin{equation}
\mathbf{Q}
=
\mathcal{G}(\mathbf{Z}),
\qquad
\mathbf{Q}\in\mathbb{R}^{h\times w\times 2},
\end{equation}
where $\mathcal{G}(\cdot)$ denotes the gate projection. Applying a softmax over the two channels yields the keep probability $\alpha_j$ for each flattened spatial location. These probabilities are normalized across the $K=hw$ spatial locations to obtain the spatial weights: 
\begin{equation}
a_j
=
\frac{\alpha_j}
{\sum_{\ell=1}^{K}\alpha_\ell},
\qquad
a_j\geq0,
\qquad
\sum_{j=1}^{K}a_j=1.
\end{equation}

The resulting weights increase the contribution of informative radar regions while reducing the influence of unstable or ambiguous responses.
The ungated variant is recovered by setting $a_j=1/K$.
The projected feature map is then flattened into 
local descriptors 
$\{\mathbf{z}_j\}_{j=1}^{K}$, where $\mathbf{z}_j\in\mathbb{R}^{64}$.
The weighted feature mean is
$\boldsymbol{\mu}=\sum_{j=1}^{K}a_j\mathbf{z}_j$,
and the spatially weighted feature-correlation representation is computed as
\begin{equation}
\mathbf{\Sigma}
=
\sum_{j=1}^{K}
a_j
(\mathbf{z}_j-\boldsymbol{\mu})
(\mathbf{z}_j-\boldsymbol{\mu})^{\top}.
\label{eq:correlation}
\end{equation}

Although Eq.~\eqref{eq:correlation}  is mathematically a weighted covariance matrix, its entries capture pairwise relationships between projected feature channels and therefore preserve co-occurrence structure that is discarded by first-order pooling. The spatial weighting further limits the contribution of unstable radar regions to the resulting representation. The weighted covariance matrix also has a useful structural property for subsequent matrix normalization.

\begin{proposition}
If $a_j\ge0$ for all $j$ and $\sum_{j=1}^{K}a_j=1$, then the matrix $\mathbf{\Sigma}$ defined in Eq.~\eqref{eq:correlation} is symmetric positive semidefinite.
\end{proposition}

\begin{proof}
Symmetry follows directly from Eq.~\eqref{eq:correlation}. For any vector $\mathbf{v}\in\mathbb{R}^{64}$,
\[
\mathbf{v}^{\top}\mathbf{\Sigma}\mathbf{v}
=
\sum_{j=1}^{N}
a_j
\left[
\mathbf{v}^{\top}
(\mathbf{z}_j-\boldsymbol{\mu})
\right]^2
\ge0,
\]
since $a_j\ge0$. Hence, $\mathbf{\Sigma}$ is positive semidefinite.
\end{proof}

This property enables stable matrix-function normalization. To improve numerical conditioning, trace normalization is first applied as
$t=\max(\operatorname{tr}(\mathbf{\Sigma}),\epsilon)$ and
$\bar{\mathbf{\Sigma}}=\mathbf{\Sigma}/t$.
Matrix square-root normalization is then computed using the Newton--Schulz iteration \cite{li2018towards}, and the normalized matrix is rescaled as
$\mathbf{M}=\bar{\mathbf{\Sigma}}^{1/2}\sqrt{t}$.
The final global descriptor is obtained by vectorizing the upper-triangular entries,

\begin{equation}
\mathbf{d}
=
\operatorname{vecu}(\mathbf{M}).
\label{eq:descriptor}
\end{equation}

Since $\mathbf{M}\in\mathbb{R}^{64\times64}$ is symmetric,
vectorizing its upper-triangular entries produces a
$D=2080$-dimensional global descriptor.

\subsection{Rotation-Robust Triplet Loss}
\label{sec:Triplet}
To improve robustness to heading variation, each training scan is circularly shifted along the azimuth axis before descriptor extraction.
Let $\mathcal{Q}_{\delta}(\cdot)$ denote an azimuthal roll by $\delta$, where $\delta$ is sampled uniformly over rotations up to $180^\circ$. For training sample $i$, the descriptor is therefore $\mathbf d_i=f(\mathcal{Q}_{\delta_i}(\mathbf X_i))$.
For each anchor $i$ in a mini-batch, batch-hard mining selects the hardest positive and hardest negative:
\begin{equation}
p_i^\ast=\arg\max_{p\in\mathcal P(i)}
\|\mathbf d_i-\mathbf d_p\|_2,
\;
n_i^\ast=\arg\min_{n\in\mathcal N(i)}
\|\mathbf d_i-\mathbf d_n\|_2 .
\end{equation}

The training objective is
\begin{equation}
\mathcal L_{\mathrm{triplet}}
=
\frac{1}{|\mathcal A|}
\sum_{i\in\mathcal A}
\left[
\|\mathbf d_i-\mathbf d_{p_i^\ast}\|_2
-
\|\mathbf d_i-\mathbf d_{n_i^\ast}\|_2
+
m
\right]_+ ,
\end{equation}
where $\mathcal{A}$ denotes the set of anchors in the mini-batch,
$\mathcal{P}(i)$ and $\mathcal{N}(i)$ are the positive and negative
sets for anchor $i$, $[x]_{+}=\max(0,x)$, and $m=0.2$ is the triplet
margin.

\section{Experimental Setup}
\label{sec:experimental_setup}

\begin{table*}[t]
\centering
\caption{Recall@1\% Radar Place Recognition on the MulRan \cite{kim2020mulran} evaluation sequences.}
\label{tab:sota}
\setlength{\tabcolsep}{12.0pt}
\renewcommand{\arraystretch}{1.05}
\begin{tabular}{lcccccccc}
\toprule
& \multicolumn{2}{c}{\textbf{Sejong}}
& \multicolumn{2}{c}{\textbf{KAIST}}
& \multicolumn{2}{c}{\textbf{Riverside}}
& \multicolumn{2}{c}{\textbf{Mean}} \\
\cmidrule(lr){2-3}
\cmidrule(lr){4-5}
\cmidrule(lr){6-7}
\cmidrule(lr){8-9}
Method
& 5\,m & 10\,m
& 5\,m & 10\,m
& 5\,m & 10\,m
& 5\,m & 10\,m \\
\midrule
Ring Key \cite{kim2018scan}
& 50.3 & 59.4
& 80.5 & 83.8
& 49.7 & 59.5
& 60.2 & 67.6 \\
ScanContext \cite{kim2018scan}
& 86.8 & 87.9
& 93.5 & 94.6
& 67.1 & 77.2
& 82.5 & 86.6 \\
VGG-16+NetVLAD \cite{komorowski2021radarloc}
& 78.9 & 93.8
& 88.9 & 93.7
& 61.3 & 83.4
& 76.4 & 90.3 \\
RadarLoc \cite{komorowski2021radarloc}
& 92.9 & 98.8
& 95.9 & 98.8
& 74.4 & 92.3
& 87.7 & 96.6 \\
\rowcolor{green!20}
\textbf{\coolname{} (Ours)}
& \textbf{97.4} & \textbf{99.9}
& \textbf{97.5} & \textbf{99.2}
& \textbf{78.9} & \textbf{95.2}
& \textbf{91.3} & \textbf{98.1} \\
\bottomrule
\end{tabular}
\end{table*}

\noindent \textbf{Dataset and Evaluation Environments.}
We evaluate \coolname{} on the MulRan dataset~\cite{kim2020mulran}, which provides $360^\circ$ radar measurements collected using a Navtech CIR204-H FMCW radar with a maximum range of 200\,m, an angular resolution of $0.9^\circ$, and a radial resolution of 0.06\,m. Each scan is represented as a polar image with 400 azimuth bins and 3,360 range bins and is downsampled to $384\times128$ pixels following~\cite{komorowski2021radarloc}. Evaluation is conducted across the Sejong, KAIST, and Riverside environments, which span urban, campus, and open-road conditions with varying scene structure, traffic, and landmark density.

\noindent \textbf{Training and Evaluation Protocol.}
We follow the training and evaluation splits introduced in~\cite{komorowski2021radarloc}. Training uses predefined regions from the Sejong 01 and 02 traversals, yielding 13,611 radar scans after pre-processing. Evaluation is conducted on geographically disjoint regions from Sejong, KAIST, and Riverside, using traversal 02 as the reference database and traversal 01 as the query sequence. This results in 1,366/1,337 database/query scans for Sejong, 3,209/3,420 for KAIST, and 2,117/2,043 for Riverside.

\noindent \textbf{Evaluation Metric.}
Performance is measured using Recall@1 (R@1). For each query, database scans are ranked by Euclidean distance between global descriptors, and retrieval is considered correct when the top-ranked scan lies within a specified distance of the ground-truth position. We report R@1 at 5\,m and 10\,m for comparison with prior work, and additionally at 3\,m to assess finer place discrimination. Results are reported for each environment and as their unweighted mean.

\noindent \textbf{Implementation Details.}
\coolname{} is implemented in PyTorch and trained using AdamW with an initial learning rate of $1\times10^{-4}$ and a batch size of 256. A step-based learning-rate schedule and random erasing are used during training. Unless otherwise stated, all compared variants use the same training data, optimization settings, and evaluation protocol.

\section{Results}

\subsection{Comparison with Existing Methods}
\label{sec:comparison}

We compare \coolname{} with representative hand-crafted and learned
radar place recognition methods evaluated under the same MulRan
protocol. Ring Key~\cite{kim2018scan} is a compact,
rotation-invariant summary derived from ScanContext, whereas
ScanContext~\cite{kim2018scan} retains the full hand-crafted polar
representation and accounts for heading changes through circular
alignment. VGG-16+NetVLAD~\cite{komorowski2021radarloc} combines a
comparatively large general-purpose convolutional backbone with
learned VLAD aggregation. RadarLoc~\cite{komorowski2021radarloc}
instead uses a compact radar-specific convolutional backbone followed by GeM pooling. Although more recent radar place recognition methods
have been proposed, many use different sensor configurations,
datasets, or evaluation protocols; we therefore use the methods
reported under this common benchmark for direct comparison.
Tab.~\ref{tab:sota} reports Recall@1 at the standard 5\,m and
10\,m thresholds.

\coolname{} achieves the highest R@1 on every sequence and at both
localization thresholds. Compared with RadarLoc, the mean R@1
increases from 87.7\% to 91.3\% at 5\,m and from 96.6\% to 98.1\%
at 10\,m. \coolname{} also substantially outperforms
VGG-16+NetVLAD despite using a smaller backbone. This
suggests that the improvement is not obtained through backbone
capacity, but through rotation-robust feature extraction and SGCA representation. The sequence-wise gains remain consistent across different operating
conditions. On Sejong, SGCA-Net improves the 5\,m Recall@1 from
92.9\% to 97.4\% despite changes between dense urban and highway-like
sections. On KAIST, performance increases from 95.9\% to 97.5\% in
the presence of repeated campus layouts. The improvement is also
maintained on Riverside, where sparse persistent landmarks and
repetitive geometry make place discrimination more difficult:
R@1 increases from 74.4\% to 78.9\% at 5\,m and from 92.3\%
to 95.2\% at 10\,m.

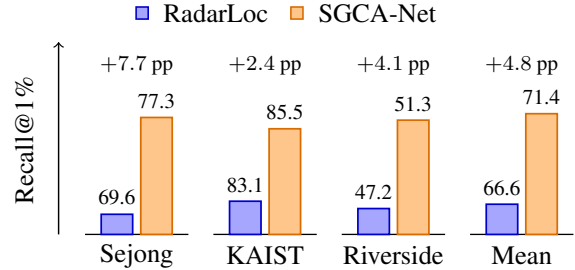
\begin{figure}[t]
    \centering

    \begin{tikzpicture}[baseline]
        \draw[
            fill=blue!35,
            draw=blue!80!black,
            line width=0.6pt
        ] (0,0) rectangle (0.16,0.16);
        \node[anchor=west, font=\normalsize]
            at (0.24,0.08) {RadarLoc };

        \draw[
            fill=orange!45,
            draw=orange!85!black,
            line width=0.6pt
        ] (2.05,0) rectangle (2.21,0.16);
        \node[anchor=west, font=\normalsize]
            at (2.29,0.08) {\coolname{}};
    \end{tikzpicture}

    \vspace{1mm}

    \begin{tikzpicture}[x=\columnwidth/7.0,y=1cm]

        \draw[->, line width=0.55pt]
            (0,0) -- (0,2.55);

        \node[
            rotate=90,
            font=\normalsize
        ] at (-0.38,1.25) {Recall@1\%};

            \begin{scope}[xshift=0.35cm]
        
            \draw[line width=0.5pt] (0,0) -- (1.15,0);

            \filldraw[
                fill=blue!35,
                draw=blue!80!black,
                line width=0.65pt
            ] (0.18,0) rectangle
              (0.53,{(69.6-68)/12*2.00});

            \filldraw[
                fill=orange!45,
                draw=orange!85!black,
                line width=0.65pt
            ] (0.62,0) rectangle
              (0.97,{(77.3-68)/12*2.00});

            \node[font=\footnotesize, anchor=south]
                at (0.355,{(69.6-68)/12*2.00}) {69.6};

            \node[font=\footnotesize, anchor=south]
                at (0.795,{(77.3-68)/12*2.00}) {77.3};

            \node[font=\footnotesize]
                at (0.575,2.24) {$+7.7$ pp};

            \node[font=\normalsize]
                at (0.575,-0.28) {Sejong};
        \end{scope}

        \begin{scope}[xshift=2.05cm]
            \draw[line width=0.5pt] (0,0) -- (1.15,0);

            \filldraw[
                fill=blue!35,
                draw=blue!80!black,
                line width=0.65pt
            ] (0.18,0) rectangle
              (0.53,{(83.1-82)/5*2.00});

            \filldraw[
                fill=orange!45,
                draw=orange!85!black,
                line width=0.65pt
            ] (0.62,0) rectangle
              (0.97,{(85.5-82)/5*2.00});

            \node[font=\footnotesize, anchor=south]
                at (0.355,{(83.1-82)/5*2.00}) {83.1};

            \node[font=\footnotesize, anchor=south]
                at (0.795,{(85.5-82)/5*2.00}) {85.5};

            \node[font=\footnotesize]
                at (0.575,2.24) {$+2.4$ pp};

            \node[font=\normalsize]
                at (0.575,-0.28) {KAIST};
        \end{scope}

        \begin{scope}[xshift=3.75cm]
            \draw[line width=0.5pt] (0,0) -- (1.15,0);

            \filldraw[
                fill=blue!35,
                draw=blue!80!black,
                line width=0.65pt
            ] (0.18,0) rectangle
              (0.53,{(47.2-46)/7*2.00});

            \filldraw[
                fill=orange!45,
                draw=orange!85!black,
                line width=0.65pt
            ] (0.62,0) rectangle
              (0.97,{(51.3-46)/7*2.00});

            \node[font=\footnotesize, anchor=south]
                at (0.355,{(47.2-46)/7*2.00}) {47.2};

            \node[font=\footnotesize, anchor=south]
                at (0.795,{(51.3-46)/7*2.00}) {51.3};

            \node[font=\footnotesize]
                at (0.575,2.24) {$+4.1$ pp};

            \node[font=\normalsize]
                at (0.575,-0.28) {Riverside};
        \end{scope}

        \begin{scope}[xshift=5.45cm]
            \draw[line width=0.5pt] (0,0) -- (1.15,0);

            \filldraw[
                fill=blue!35,
                draw=blue!80!black,
                line width=0.65pt
            ] (0.18,0) rectangle
              (0.53,{(66.6-65)/8*2.00});

            \filldraw[
                fill=orange!45,
                draw=orange!85!black,
                line width=0.65pt
            ] (0.62,0) rectangle
              (0.97,{(71.4-65)/8*2.00});

            \node[font=\footnotesize, anchor=south]
                at (0.355,{(66.6-65)/8*2.00}) {66.6};

            \node[font=\footnotesize, anchor=south]
                at (0.795,{(71.4-65)/8*2.00}) {71.4};

            \node[font=\footnotesize]
                at (0.575,2.24) {$+4.8$ pp};

            \node[font=\normalsize]
                at (0.575,-0.28) {Mean};
        \end{scope}

    \end{tikzpicture}

    \caption{R@1 at the stricter 3\,m localization threshold. Values
    above the bars report R@1, while annotations show the absolute
    improvement of \coolname{} over RadarLoc in percentage points.}
    \label{fig:three_meter}
\end{figure}

\begin{table*}[t]
\centering
\caption{Ablation study of the proposed SGCA block on MulRan. Entries
report R@1\% at 3\,m, 5\,m, and 10\,m. Mean is computed across
the Sejong, KAIST, and Riverside sequences. The full \coolname{} configuration is
highlighted.}
\label{tab:ablation}
\setlength{\tabcolsep}{6.7pt}
\renewcommand{\arraystretch}{1.05}
\resizebox{\textwidth}{!}{
\begin{tabular}{lcccccccccccc}
\toprule
& \multicolumn{3}{c}{Sejong}
& \multicolumn{3}{c}{KAIST}
& \multicolumn{3}{c}{Riverside}
& \multicolumn{3}{c}{Mean} \\
\cmidrule(lr){2-4}
\cmidrule(lr){5-7}
\cmidrule(lr){8-10}
\cmidrule(lr){11-13}
Method
& 3\,m & 5\,m & 10\,m
& 3\,m & 5\,m & 10\,m
& 3\,m & 5\,m & 10\,m
& 3\,m & 5\,m & 10\,m \\
\midrule
SGCA-Net w/ GeM pooling
& 69.6 & 93.0 & 98.8
& 83.1 & 95.8 & 98.8
& 47.2 & 74.3 & 92.3
& 66.6 & 87.7 & 96.6 \\

SGCA-Net w/o gate
& 73.9 & 95.5 & 99.5
& 84.2 & 96.5 & 98.9
& 50.1 & 76.1 & 94.3
& 69.4 & 89.4 & 97.6 \\

SGCA-Net w/ standard projection
& 75.0 & 96.3 & 99.5
& 84.0 & 96.6 & 99.0
& 49.5 & 77.5 & 94.2
& 69.5 & 90.1 & 97.6 \\

\rowcolor{green!20}
\textbf{SGCA-Net}
& \textbf{77.3} & \textbf{97.4} & \textbf{99.9}
& \textbf{85.5} & \textbf{97.5} & \textbf{99.2}
& \textbf{51.3} & \textbf{78.9} & \textbf{95.2}
& \textbf{71.4} & \textbf{91.3} & \textbf{98.1} \\
\bottomrule
\end{tabular}}
\end{table*}

We additionally evaluate performance at the stricter
3\,m threshold, as shown in Fig.~\ref{fig:three_meter}. Since this
threshold was not reported in the original RadarLoc study, we reproduce
the RadarLoc results using its public implementation. \coolname{}
improves R@1 across all three evaluation environments, with gains of
7.7\%, 2.4\%, and 4.1\% on Sejong, KAIST, and Riverside,
respectively. Overall, the mean R@1 increases from 66.6\% to 71.4\%. 
These results show that \coolname{} provides reliable
discrimination when geographically close places produce similar radar observations.

\subsection{Ablation Analysis}
\label{sec:ablation_analysis}

We analyse the contribution of the \coolname{} components and evaluate
robustness to heading variation.

\paragraph{\coolname{} Component Ablation.}
We conduct an ablation study to quantify the contribution of each
\coolname{} component, as reported in Tab.~\ref{tab:ablation}. The GeM
variant replaces the SGCA block with generalized mean pooling \cite{radenovic2018fine}, a widely
used global aggregation method also adopted by RadarLoc. This
variant highlights the effect of replacing GeM with SGCA while retaining
the same radar-specific backbone and training configuration. The
ungated variant retains feature-correlation aggregation but assigns
uniform weights to all feature-map locations. The standard-projection
variant includes the learned keep/dustbin gate but replaces the residual
bottleneck projection with a single projection layer. All variants use
the same training data, optimization settings, and evaluation protocol.

Replacing GeM pooling with ungated feature-correlation aggregation
increases the mean R@1 from 66.6\% to 69.4\% at 3\,m, from 87.7\% to
89.4\% at 5\,m, and from 96.6\% to 97.6\% at 10\,m. These results
demonstrate the benefit of modelling feature correlations even when all
feature-map locations are weighted uniformly. Adding the keep/dustbin
gate and residual bottleneck projection yields further improvements.
The complete \coolname{} achieves the highest mean R@1 at all
thresholds, reaching 71.4\%, 91.3\%, and 98.1\% at 3\,m, 5\,m, and
10\,m, respectively.

\paragraph{Robustness to Heading Variation.}
We evaluate \coolname{} under random query rotations of up to
$180^{\circ}$, following the rotation-enabled setting used by RadarLoc.
As reported in Tab.~\ref{tab:rotation}, \coolname{} improves the mean
R@1 over RadarLoc from 65.9\% to 70.6\% at 3\,m, from 87.4\% to
91.1\% at 5\,m, and from 96.5\% to 98.1\% at 10\,m. The 
consistent improvements show that the SGCA components strengthen place
discrimination while retaining robustness to substantial heading
changes.

\begin{table}[t]
\centering
\caption{Mean R@1\% under random query rotations of up to
$180^{\circ}$.}
\label{tab:rotation}
\setlength{\tabcolsep}{6pt}
\renewcommand{\arraystretch}{1.05}
\begin{tabular}{lccc}
\toprule
Method & 3\,m & 5\,m & 10\,m \\
\midrule
RadarLoc & 65.9 & 87.4 & 96.5 \\
\rowcolor{green!20}
\textbf{SGCA-Net}
& \textbf{70.6}
& \textbf{91.1}
& \textbf{98.1}\\
\bottomrule
\end{tabular}
\end{table}

 \subsection{Transfer to Unseen Radar Environments}
\label{sec:cross_benchmark}

Generalization to unseen environments remains a challenge for place recognition, with prior work investigating test-time adaptation to compensate for domain shift~\cite{geoadapt2024}. Here, we instead examine whether the representation learned by \coolname{} transfers beyond the MulRan training environments without adaptation or fine-tuning.
transfers beyond the MulRan training environments. 
Following the spinning radar evaluation protocol used in~\cite{kim2025sherloc}, \coolname{} is trained only on MulRan
and evaluated without fine-tuning on the Sports Complex and Library
sequences of HeRCULES \cite{kim2025hercules}. RadarLoc~\cite{komorowski2021radarloc} is evaluated under the same setting to
provide a consistent reference. The session pairs include
clear-to-night/cloudy and clear-to-snowy changes. We report R@1 for
each pair, while AR@1 denotes the mean R@1 across all four pairs.

\begin{table}[t]
    \centering
    \caption{Transfer evaluation on the HeRCULES \cite{kim2025hercules} spinning-radar splits
    used in~\cite{kim2025sherloc}. Both methods are trained
    only on MulRan and evaluated without fine-tuning. AR@1 denotes the
    mean R@1 across the four session pairs.}
    \label{tab:cross_benchmark}
    \setlength{\tabcolsep}{3.0pt}
    \renewcommand{\arraystretch}{1.05}
    \begin{tabular}{lccccc}
        \toprule
        & \multicolumn{2}{c}{Sports Complex}
        & \multicolumn{2}{c}{Library} & \\
        \cmidrule(lr){2-3}
        \cmidrule(lr){4-5}
        Method
        & \rotatebox{45}{01$\rightarrow$02}
        & \rotatebox{45}{01$\rightarrow$03}
        & \rotatebox{45}{01$\rightarrow$02}
        & \rotatebox{45}{01$\rightarrow$03}
        & AR@1 \\
        \midrule
        RadarLoc~\cite{komorowski2021radarloc}
        & 0.965 & 0.938 & 0.989 & 0.867 & 0.940 \\
        \rowcolor{green!20}
        \textbf{\coolname{}}
        & \textbf{0.986}
        & \textbf{0.973}
        & \textbf{0.993}
        & \textbf{0.880}
        & \textbf{0.958} \\
        \bottomrule
    \end{tabular}
\end{table}

As shown in Tab.~\ref{tab:cross_benchmark}, \coolname{} outperforms
RadarLoc on all four HeRCULES session pairs and increases AR@1 from
0.940 to 0.958. The largest improvement occurs on Sports Complex
01$\rightarrow$03, where R@1 increases from 0.938 to 0.973. These
results suggest that the representation learned by \coolname{} retains
its advantage when transferred to previously unseen environments and
conditions without additional fine-tuning.

\subsection{Qualitative Analysis}
\label{sec:qualitative}

Fig.~\ref{fig:qualitative} presents representative Sejong examples
for which RadarLoc returns an incorrect top-ranked match, whereas
\coolname{} identifies the correct place under the 3\,m criterion.
The selected cases include broad curved returns, asymmetric scene
structure, and open-road environments with limited distinctive
landmarks. In each case, geographically different locations produce
similar overall radar-response distributions, making them difficult to
distinguish using conventional global aggregation.

\coolname{} correctly recognizes all three places, suggesting that
spatially gated feature-correlation aggregation preserves scene
relationships that help distinguish structurally similar observations.
These examples are consistent with the improvements observed at the
stricter 3\,m threshold, where accurate recognition requires finer
discrimination between nearby and visually similar places.

\begin{figure}[t]
    \centering
    \includegraphics[width=1\columnwidth]{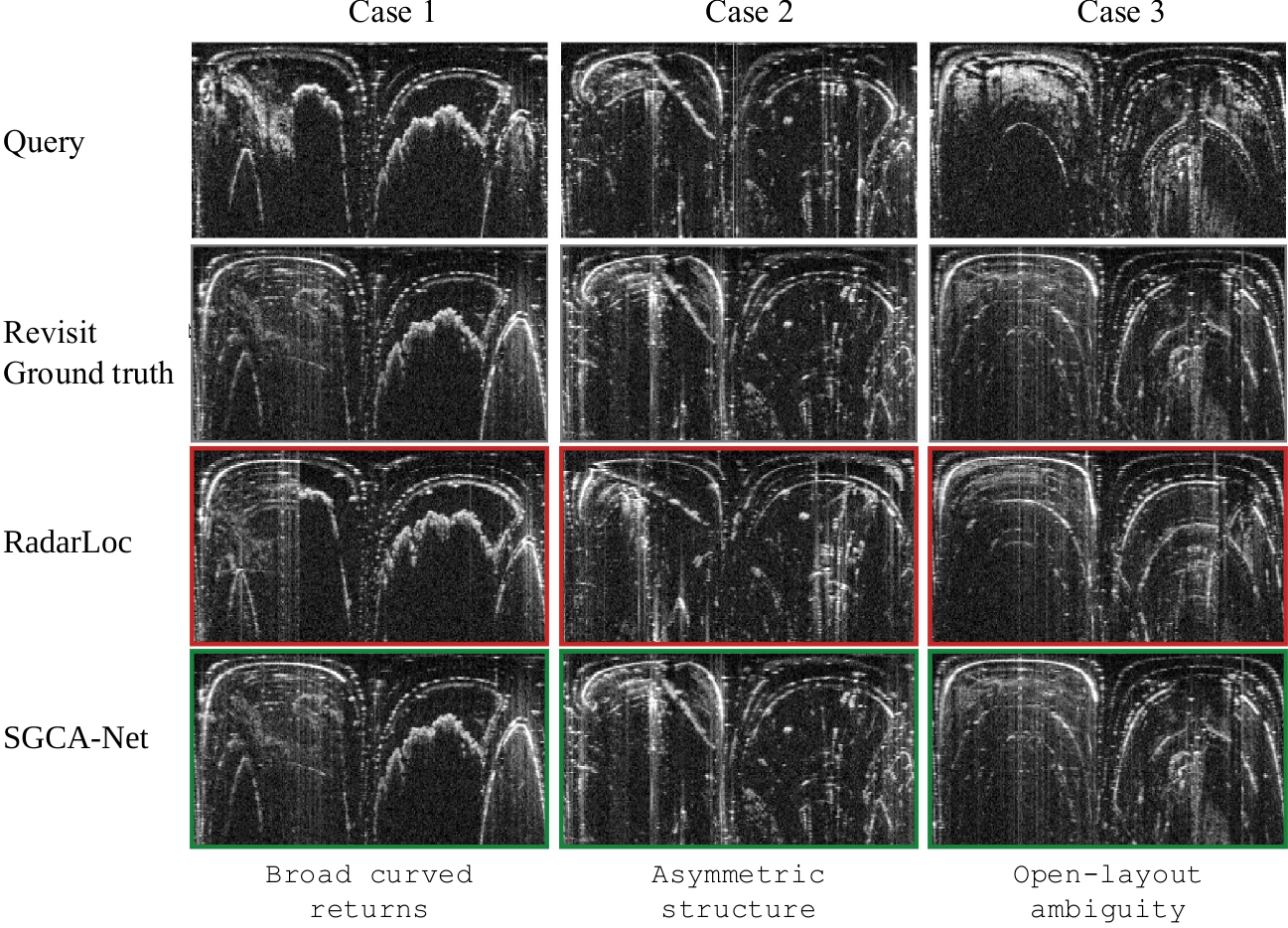}
    \caption{
    Representative retrieval results under the 3\,m
    criterion. For each case, the rows show the query, the
    ground-truth revisit, and the top-ranked results from RadarLoc and
    SGCA-Net. Green and red borders denote correct and incorrect
    retrievals, respectively. RadarLoc is confused by structurally
    similar alternatives, whereas SGCA-Net retrieves the correct
    revisit.}
    \label{fig:qualitative}
\end{figure}

\section{Conclusion and Future Work}
\label{sec:conclusion}

This paper introduced \coolname{}, a spinning radar place recognition framework that combines rotation-robust feature extraction with spatially gated feature-correlation aggregation. By modelling pairwise relationships among radar features while reducing the influence of unstable and ambiguous regions, \coolname{} provides a more discriminative global representation than conventional first-order aggregation. Experiments on MulRan demonstrate consistent improvements over representative RPR methods across diverse environments and under substantial heading variation. Evaluation on HeRCULES further demonstrates that \coolname{} generalizes to unseen environments and radar sensors without fine-tuning. Future work will investigate more computationally efficient feature-correlation representations and extend the evaluation to additional radar sensors, datasets, and long-term environmental changes. We also plan to explore uncertainty-aware spatial weighting to better account for transient and unreliable radar responses.

\section{Acknowledgments}
The authors would like to acknowledge support from the CSIRO's Generative AI for Radar Self-Supervised Learning project in partnership with Boeing.

\balance{}

{\small
\bibliographystyle{IEEEtran}
\bibliography{references}
}

\end{document}